\documentclass[letterpaper]{article} 
\usepackage{aaai24}  
\usepackage{times}  
\usepackage{helvet}  
\usepackage{courier}  
\usepackage[hyphens]{url}  
\usepackage{graphicx} 
\usepackage{natbib}  
\usepackage{caption} 
\usepackage{algorithm}
\usepackage{algorithmic}
\usepackage{multirow}
\usepackage{enumitem, array}

\usepackage{amsmath,bm}

\def\eqref#1{equation~\ref{#1}}

\def\1{\bm{1}}

\DeclareMathAlphabet{\mathsfit}{\encodingdefault}{\sfdefault}{m}{sl}
\SetMathAlphabet{\mathsfit}{bold}{\encodingdefault}{\sfdefault}{bx}{n}

\DeclareMathOperator*{\argmax}{arg\,max}

\newcommand{\ols}[1]{\mskip.5\thinmuskip\overline{\mskip-.5\thinmuskip {#1} \mskip-.5\thinmuskip}\mskip.5\thinmuskip} 
\newcommand{\olsi}[1]{\,\overline{\!{#1}}} 
\makeatletter
\newcommand\closure[1]{
  \tctestifnum{\count@stringtoks{#1}>1} 
  {\ols{#1}} 
  {\olsi{#1}} 
}
\long\def\count@stringtoks#1{\tc@earg\count@toks{\string#1}}
\long\def\count@toks#1{\the\numexpr-1\count@@toks#1.\tc@endcnt}
\long\def\count@@toks#1#2\tc@endcnt{+1\tc@ifempty{#2}{\relax}{\count@@toks#2\tc@endcnt}}
\def\tc@ifempty#1{\tc@testxifx{\expandafter\relax\detokenize{#1}\relax}}
\long\def\tc@earg#1#2{\expandafter#1\expandafter{#2}}
\long\def\tctestifnum#1{\tctestifcon{\ifnum#1\relax}}
\long\def\tctestifcon#1{#1\expandafter\tc@exfirst\else\expandafter\tc@exsecond\fi}
\long\def\tc@testxifx{\tc@earg\tctestifx}
\long\def\tctestifx#1{\tctestifcon{\ifx#1}}
\long\def\tc@exfirst#1#2{#1}
\long\def\tc@exsecond#1#2{#2}
\makeatother
\usepackage{newfloat}
\usepackage{listings}
\usepackage{tikz}
\usepackage{comment}
\usepackage{amsmath,amssymb} 
\usepackage{cite}
\usepackage{subcaption}
\usepackage{tabularray}
\usepackage{booktabs}
\usepackage[utf8]{inputenc} 
\usepackage{url}            
\usepackage{amsfonts}       
\usepackage{nicefrac}       
\usepackage{microtype} 
\usepackage{xspace}

\def\etal{{et al.\xspace}}

\def\ood{\textsc{ood}\xspace}

\def\id{\textsc{id}\xspace}

\def\name{$\mathcal{S}$NN\xspace}

\def\*#1{\mathbf{#1}}

\usepackage{url} 
\usepackage{mathtools}
\usepackage{amsthm}

\theoremstyle{plain}
\newtheorem{theorem}{Theorem}[section]

\newtheorem{lemma}[theorem]{Lemma}

\theoremstyle{definition}

\theoremstyle{remark}

\usepackage{xspace}
\usepackage{array}
\usepackage{ragged2e}
\newcolumntype{P}[1]{>{\RaggedRight\hspace{0pt}}p{#1}}
\newcolumntype{X}[1]{>{\RaggedRight\hspace*{0pt}}p{#1}}
\usepackage{tcolorbox}
\usepackage{tikz}
\usetikzlibrary{arrows,shapes,positioning,shadows,trees,mindmap}
\usepackage[edges]{forest}
\usetikzlibrary{arrows.meta}
\colorlet{linecol}{black!75}
\usepackage{xkcdcolors}

\usetikzlibrary{backgrounds}
\usetikzlibrary{arrows,shapes}
\usetikzlibrary{tikzmark}
\usetikzlibrary{calc}

\usepackage[capitalize,noabbrev]{cleveref}
\usepackage{newfloat}
\usepackage{listings}
\DeclareCaptionStyle{ruled}{labelfont=normalfont,labelsep=colon,strut=off} 
\floatstyle{ruled}
\newfloat{listing}{tb}{lst}{}
\floatname{listing}{Listing}
\title{How to Overcome Curse-of-Dimensionality for Out-of-Distribution Detection?}

\author{
    Soumya Suvra Ghosal\footnote{Equal contributions},
    Yiyou Sun*,
    Yixuan Li
}
\affiliations{
    Department of Computer Sciences, University of Wisconsin -- Madison\\
    \{sghosal, sunyiyou, sharonli\}@cs.wisc.edu
   
}
\usepackage{bibentry}

\begin{document}

\maketitle

\begin{abstract}
Machine learning models deployed in the wild can be challenged by out-of-distribution (OOD) data from unknown classes. Recent advances in OOD detection rely on distance measures to distinguish samples that are relatively far away from the in-distribution (ID) data. Despite the promise, distance-based methods can suffer from the curse-of-dimensionality problem, which limits the efficacy in high-dimensional feature space. 
To combat this problem, we propose a novel framework, Subspace Nearest Neighbor (\name), for OOD detection. In training, our method regularizes the model and its feature representation by leveraging the most relevant subset of dimensions (i.e. subspace).  Subspace learning yields highly distinguishable distance measures between ID and OOD data. 
We provide comprehensive experiments and ablations to validate the efficacy of \name. 
Compared to the current best distance-based method, \name reduces the average FPR95 by $15.96\%$ on the CIFAR-100 benchmark.
\end{abstract}
\vspace{-0.3cm}
\section{Introduction}
\label{sec:intro}

Modern machine learning models deployed in the wild face risks from out-of-distribution (OOD) data -- samples from novel contexts and classes that were not taught to the model during training. Identifying OOD samples is paramount for safely operating machine learning models in uncertain environments. A safe system should be able to identify the unknown data as OOD so that the control can be passed to humans. This illuminates the importance of OOD detection, which allows the learner to express ignorance and take further precautions.

Recent solutions for OOD detection tasks often rely on distance measures to distinguish samples that are conspicuously dissimilar from the in-distribution (ID) data~\cite{lee2018simple, tack2020csi, 
 2021ssd, sun2022knn}. These methods leverage embeddings extracted from a well-trained classifier and operate under the assumption that test OOD samples are far away from the ID data in the feature space. For example, \citet{sun2022knn} recently showed that non-parametric nearest neighbor distance is a strong indicator function for OOD likelihood. Despite the promise, distance measures can be sensitive to the dimensionality of the feature space. Those research questions have been raised in the 90s,
including \citet{beyer1999nearest}, as to whether the nearest neighbor
is meaningful in high dimensions. The key result of ~\cite{beyer1999nearest} states that in high dimensional spaces, the difference between the minimum and the maximum distance between a random reference point and a list of $n$ random data points $\*x_1$,...,$\*x_n$ become indiscernible compared to the minimum distance:
\begin{equation}
    \lim_{d \rightarrow \infty} E\bigg( \frac{D_\text{max} - D_\text{min}}{D_\text{min}}\bigg) \rightarrow 0.
\end{equation}

This phenomenon---recognized as one aspect of the ``\textit{curse of dimensionality}''~\cite{beyer1999nearest, hinneburg2000nearest, aggarwal2001surprising, houle2010can, kriegel2009outlier}---limits the efficacy of distance-based OOD detection method for modern neural networks. For example, let's consider 50000 points sampled randomly from a 2048 dimensional space bounded by $[0,1]$. Given a test sample, to find the 10-th nearest neighbor, we need a hyper-cube with a length of at least 0.9958\footnote{The calculation is done by $0.9958^{2048} \approx \frac{10}{50000}$. }, which covers the entire feature space. In other words, the high-dimensional spaces can lead to instability of nearest neighbor queries. Despite its importance, the problem has received little 
attention in the literature on OOD detection. This begs the following question:

\begin{center}
    \textit{How do we combat the curse-of-dimensionality for OOD detection?}
\end{center}
Targeting this important problem, we propose a new framework called \emph{Subspace Nearest Neighbor} (\textbf{SNN}) for OOD detection. Our key idea is to learn a feature subspace---a subset of relevant features---for distance-based detection. Our method is motivated by the observation drawn in \citet{houle2010can}: irrelevant attributes in a
feature space may impede the separability of different distributions and thus have the
potential for rendering the nearest neighbor less meaningful.
In light of this,
\name regularizes the model and its feature representation by leveraging the most relevant subset of dimensions (\emph{i.e.} subspace) for the class prediction. Geometrically, this is equivalent to projecting the embedding to the selected subspace, and the output for
each class is derived using the projected features. The subspaces are learned in a way that maintains the discriminative power for classifying ID data. During testing, distance-based OOD detection is performed by leveraging the learned features.

We show that \name is both theoretically grounded (Section~\ref{sec:theory}) and empirically effective (Section~\ref{sec:experiment}) to combat the curse-of-dimensionality for OOD detection. 
Extensive experiments show that \name substantially outperforms the competitive methods in the literature.  For example,
using the CIFAR-100 dataset as ID and LSUN~\cite{yu2015lsun} as OOD data,
our approach reduces the FPR95 from {66.09}\% to {24.43}\%---a \textbf{41.66}\% direct improvement over the current best distance-based  method KNN~\cite{sun2022knn}. Beyond the OOD detection task, our subspace learning algorithm also leads to improved calibration performance of the ID data itself\footnote{Code is available at \url{https://github.com/deeplearning-wisc/SNN}}. We summarize our contributions below:

\begin{enumerate}
    \item We propose a new framework, \emph{subspace nearest neighbor} (\name),  for distance-based OOD detection. \name effectively alleviates the curse-of-dimensionality suffered by the vanilla KNN approach operating on the full feature space.
    
    \item We demonstrate the strong performance of \name on an extensive collection of OOD detection tasks. On CIFAR-100, \name substantially reduces the average FPR95 by \textbf{15.96}\% compared to the current best distance-based approach~\cite{sun2022knn}. Further, we provide in-depth analysis and ablations to understand the effect of each component in our method (Section~\ref{sec:ablations}). 
    
    \item We provide theoretical analysis, showing that ID and OOD data become more distinguishable in a subspace with reduced dimensions. We hope our work draws attention to the strong promise of subspace learning for OOD detection.

\end{enumerate}

\section{Preliminaries}
\label{sec:background}
\paragraph{Setup.} In this paper, we consider a supervised multi-class classification problem, where $\mathcal{X}$ denotes the input space and $\mathcal{Y}=\{1,2,...,C\}$ denotes the label space. The training set $\mathbb{D}_\text{in} = \{(\*x_i, y_i)\}_{i=1}^N$ is drawn \emph{i.i.d.} from the joint data distribution $P_{\mathcal{X}\mathcal{Y}}$.
Let $\mathcal{P}_\text{in}$ denote the marginal distribution on $\mathcal{X}$. Let $f: \mathcal{X} \mapsto \mathbb{R}^{|\mathcal{Y}|}$ be a neural network trained on samples drawn from $P_{\mathcal{X}\mathcal{Y}}$ to output a logit vector, which is used to predict the label of an input sample. 

\vspace{0.05cm}
\noindent\textbf{Out-of-distribution detection.} 
Our framework concerns a common real-world scenario in which the algorithm is trained on the ID data but will then be deployed in environments containing \emph{out-of-distribution} samples from {unknown} class $y\notin \mathcal{Y}$ and therefore should not be predicted by $f$. Formally, OOD detection can be formulated as a level-set estimation problem. At test time, one can perform the  following test to determine whether a sample $\*x \in \mathcal{X}$ is from $\mathcal{P}_{\text{in}}$ (ID) or not (OOD):
\begin{equation}
	G_{\lambda}(\*x) =
	\begin{cases}
		\id & \text{if}\ S(\*x) \geq \lambda, \\
		\ood & \text{if}\ S(\*x) < \lambda \\
	\end{cases}
	\label{equ:ood_score}
\end{equation}
where $S(\*x$) denotes a scoring function and $\lambda$ is a chosen threshold such that a high fraction ($95\%$) of \id data is correctly classified. Test samples with higher values of S($\*x$) are classified as \id and vice-versa.

\section{Methodology}
\label{sec:method}

A common design of $S(\*x)$ often relies on distance measures~\cite{2021ssd,sun2022knn,lee2018simple} to distinguish OOD samples that are far away from the ID data in the feature space. Despite the promise, they operate on the full feature space with a large dimension, which is known to be susceptible to the ``{curse-of-the dimensionality}''. Intuitively, irrelevant attributes in a
feature space may impede the separability of different distributions and thus have the
potential to render the distance measure less meaningful. Rather than relying on the full feature space, our key idea is to use a subset of relevant dimensions for distance-based OOD detection.

In what follows, we formally introduce our framework, \emph{Subspace Nearest Neighbor} (dubbed {\name}) for OOD detection. \name consists of two components. First, during training, the model learns the relevant subset of dimensions
(\emph{i.e.} subspace) for each class. The subspaces are learned in a way that maintains the discriminative power for classifying ID data (Section~\ref{sec:train}). Second, during testing, distance-based OOD detection is performed by using the learned features (Section~\ref{sec:test}).  

\subsection{Learning the Subspaces}
\label{sec:train}

We start with the first challenge: {how to learn the relevant subset of dimensions (a.k.a. subspace)?} Our idea is driven by the fact that the class prediction of a given example depends only on a subset of relevant features. Moreover, the relevance of any particular feature dimension may vary across different classes and instances. For example, for a \textsc{Dog} class instance, features such as tail, nose, etc are relevant. However, these feature dimensions can be irrelevant for an instance from \textsc{Bird} class. This necessitates training the model using \emph{class-dependent subspaces} conditioned on the given instance. We formally describe the learning procedure in the sequel.

\paragraph{Defining the subspace.} We consider a deep neural network parameterized by $\theta$, which encodes an input $\*x \in \mathcal{X}$ to a feature space with dimension $m$. We denote by $ h(\*x) = [h^{(1)}(\*x), h^{(2)}(\*x),...,h^{(m)}(\*x) ]$ the feature vector from the penultimate layer of the network, which is a $m$-dimensional vector. For a class $c$, we select the subspace by masking the original feature:
\begin{equation}
\small
    h_{\text{masked}}(\*x) = h(\*x) \odot R_c(\*x), \textit{s.t.}, R_c(\*x)\in \{0,1\}^m, \|R_c(\*x)\|_1 = s,
\end{equation}
where $\odot$ represents the element-wise multiplication and  $R_c(\*x)$  is a binary mask with $s$ non-zero elements encoding the subset of dimensions (\emph{i.e.} subspace) for class $c$. This way, the high-dimensional features can be projected onto the corresponding subspace defined by $R_c(\*x)$, which is class-dependent for a given input $\*x$. The model output $f_c(\*x;\theta, R_c(\*x))$ is further obtained by passing the projected feature through a linear transformation with weight vector $\*w_c \in \mathbb{R}^m$:
\begin{equation}
\label{eq:orig}
    f_c(\*x; \theta, R_c(\*x)) = \langle \*w_c, h(\*x) \odot R_c(\*x) \rangle .
\end{equation}

\paragraph{Learning subspace via relevance.}
Given an input $\*x$, we now describe how to identify $R_c(\*x)$, the subset of most \emph{relevant} dimensions for a class $c$. Our key idea is to formulate the following  optimization objective:
\begin{equation}
    \min_\theta ~~\mathbb{E}_{(\*x,y)\sim\mathcal{P}_{\mathcal{X}\mathcal{Y}}} \left[\mathcal{L}_{\text{CE}}\left({{f}}(\*x;\theta, R_c(\*x)), y\right)\right],
\end{equation}
\begin{multline}
     \text{where ~~~}f_c(\*x; \theta, R_c(\*x)) \\
    = \max_{R_c(\*x)\in \{0,1\}^m, \|R_c(\*x)\|_1 = s} \langle \*w_c, h(\*x) \odot R_c(\*x) \rangle. 
    \label{eq:relevance}
\end{multline}

Equation~\ref{eq:relevance} indicates that the subspace is chosen based on the features that are most responsible for the class prediction. 
Geometrically, this is equivalent to projecting 
the point to the selected subspace, and the output for each class is derived using
the projected features. The search of subspace can be computed efficiently. One can easily prove that $f_c(\*x;\theta, R_c(\*x))$ in Equation~\ref{eq:relevance} is equivalent to the summation of the top $s$ entries in 
$\*w_c \odot h(\*x)$ with the highest values. In Section~\ref{subsec:common_benchmark}, we show that our subspace learning objective incurs similar (and in some cases faster) training time compared to the standard cross-entropy loss. During inference, given a test sample $\mathbf{x}^{\prime}$, the class prediction $\hat{y}$ is made by $
    \hat{y}  = \argmax_{c \in \mathcal{Y}} \langle \*w_c, h(\*x^{\prime}) \odot R_c(\*x^{\prime}) \rangle $.

Our formulation also flexibly allows each feature dimension to be used for multiple class predictions. As an example, a class \textsc{laptop} might depend on the two most relevant feature dimensions: \textsc{Keyboard}  and \textsc{screen}. Similarly, the class prediction for \textsc{TV} might be relying on two features \textsc{Telecontrol}  and \textsc{screen}. In both cases, the feature of \textsc{screen} is
responsible for both TV and Laptop. 

\subsection{Out-of-Distribution Detection with Subspace Nearest Neighbor}
\label{sec:test}

We now describe how to perform test-time OOD detection leveraging the learned features.
In particular, given a test sample's embedding $\*z'$,  we determine its OOD-ness by computing the  $k$-th nearest neighbor distance \textit{w.r.t.} the training embeddings.  Here $\*z'= h(\*x') / \lVert h(\*x') \rVert_2$ is the $L_2$-normalized feature embedding. The decision function for OOD detection is given by: 
\begin{equation*}
    G(\*x';k, \mathbb{D}_\text{in}) = \mathbf{1}\{-r_k(\*z') \ge \lambda\},
\end{equation*} 
where $r_k(\*z') = ||\*z' - \*z_{k}||_2$ is the  distance to the $k$-th nearest neighbor ($k$-NN) training embedding ($\*z_k$) and $\mathbf{1}\{\cdot\}$ is the indicator function. The threshold $\lambda$ is typically chosen so that a high fraction of ID data (\emph{e.g.,} 95\%) is correctly classified. Different from \citet{sun2022knn}, we calculate the nearest neighbor distance based on the subspace-regularized feature space, rather than the vanilla ERM-trained model. Next, we  show that our subspace nearest neighbor approach is both theoretically grounded (Section~\ref{sec:theory}) and empirically effective (Section~\ref{sec:experiment}) to combat the curse-of-dimensionality for OOD detection.

\section{Theoretical Insights}
\label{sec:theory}
In this section, we provide theoretical insights into why using a feature subspace is critical for $k$-NN distance-based OOD detection.

\noindent \textbf{Setup.}
We consider the OOD detection task as a special binary classification task, where the negative samples (OOD) are only available in the testing stage. A key challenge in OOD detection (and theoretical analysis) is the lack of knowledge on OOD distribution. Common algorithms detect OOD samples when ID density  $p(\*z_i|\*z_i \in  \text{ID})$ is low. For simplicity, we let $p_{in}(\*z_i) = p(\*z_i|\*z_i \in \text{ID})$. We thus divide the bounded-set $\mathcal{Z}$ into two disjoint set $\mathcal{Z}_{in}$ and $\mathcal{Z}_{out}$ with:
    \begin{align}
        \mathcal{Z}_{in} = \{\*z | p_{in}(\*z) \geq \lambda\},  \mathcal{Z}_{out} = \{\*z | p_{in}(\*z) < \lambda\}, 
    \end{align}
where $\lambda$ is a certain real number.

\noindent \textbf{Main result. }   In our analysis, we aim to investigate the impact of embedding dimension $m$ on the following gap:
\begin{align}
    \hat{\Delta}(m) = \mathbb{E}[\hat{p}_{in}(\*z)|{\*z \in \mathcal{Z}_{in}}] - \mathbb{E}[\hat{p}_{in}(\*z)|{\*z \in \mathcal{Z}_{out}}] ,
\end{align}
which is the difference of average estimated density between $\mathcal{Z}_{in}$ and $\mathcal{Z}_{out}$. This gap can be translated into OOD detection performance with the $k$-NN distance. To see this, a function of the $k$-NN distance $r_k(\*z)$ can be regarded as an estimator which approximates $p_{in}(\*z)$ \cite{zhao2022analysis}:
\begin{equation}
  \hat{p}_{in}(\*z) = \frac{k-1}{N V(B(\mathbf{z}, r_k(\*z))},
\end{equation}
where $V(B(\mathbf{z}, r_k(\mathbf{z})))$ denotes the volume of the ball $B(\mathbf{z}, r_k(\mathbf{z}))$ centered at $\mathbf{z}$ with radius $r_k(\mathbf{z})$ and $N$ represents number of training samples. To investigate how OOD detection performance with $k$-NN distance is impacted by the {curse-of-dimensionality}, we derive the following theorem showing how the lower bound of $\hat{\Delta}(m)$ changes when $m$ becomes large. 

\begin{theorem} Let $\mathbb{E}[p_{in}(\*z)|{\*z \in \mathcal{Z}_{in}}] - \mathbb{E}[p_{in}(\*z)|{\*z \in \mathcal{Z}_{out}}] = \Delta(m)$ be a function of the feature dimensionality $m$. We have the following bound (proof is in Appendix B): 
\begin{align}
            \hat{\Delta}(m) & \geq \Delta(m) - 
            O\left((\frac{k}{N})^{\frac{1}{m}} + k^{-\frac{1}{2}}\right).
    \end{align}

\label{th:main}
\end{theorem}

\noindent \textbf{Implications of theory.} 
In Theorem~\ref{th:main}, we see that $\hat{\Delta}(m)$ is lower-bounded by two terms: $\Delta(m)$ and $-O\left((\frac{k}{N})^{\frac{1}{m}} + k^{-\frac{1}{2}}\right)$. Specifically, $\Delta(m)$ represents the true density gap between the ID set and OOD set, which generally sets a ground-truth oracle on how well we can perform OOD detection.  The second term $-O\left((\frac{k}{N})^{\frac{1}{m}} + k^{-\frac{1}{2}}\right)$ measures the approximation error 
when using $k$-NN distance to estimate the true density gap. The approximation error worsens when $m$ gets larger,  which is { exactly where the ``curse of dimensionality'' comes from}.  Consider the following case: when $\Delta(m)$ is nearly constant, $-O\left((\frac{k}{N})^{\frac{1}{m}} + k^{-\frac{1}{2}}\right)$  monotonically decreases when $m$ gets larger, then it is likely for $\hat{\Delta}(m)$ goes to 0 (indistinguishable between ID \& OOD). This theoretical finding further validates the necessity of using a subspace of embedding, which is exactly the solution provided by $\mathcal{S}$NN.

\vspace{0.2cm}
\noindent \textbf{Remark on a small $\mathbf{m}$.} If we only focus on the second term $-O\left((\frac{k}{N})^{\frac{1}{m}} + k^{-\frac{1}{2}}\right)$, it seems like we need the dimension $m$ to be as small as possible. However, in practice, a very small dimension will lose the necessary information needed to separate ID and OOD. Such implications are captured by $\Delta(m)$ which measures the true density gap between ID and OOD. If $\Delta(m)$ is an increasing function over $m$, there is a trade-off in choosing $m$ with the second term that inversely encourages a smaller $m$. Such insight is indeed verified empirically in Section~\ref{sec:ablations} that there exists an optimal $m$--- neither too small nor too large.

\section{Experiments}
\label{sec:experiment}

\begin{table}[t]
\centering
\small
\resizebox{0.99\linewidth}{!}{
\begin{tabular}{lcccc}
\multirow{2}{1.5cm}{\textbf{Methods}} & \multicolumn{2}{c}{\textbf{Far-OOD}} & \multicolumn{2}{c}{\textbf{Near-OOD}}\\
\cmidrule{2-5}
& \textbf{FPR95}  & \textbf{AUROC} & \textbf{FPR95}  & \textbf{AUROC}\\
& $\downarrow$ & $\uparrow$ & $\downarrow$ & $\uparrow$ \\
\toprule
\emph{Using model outputs}\\
MSP~\cite{hendrycks2016baseline} & 52.11 & 91.79 & 64.66 & 85.28 \\
ODIN~\cite{liang2018enhancing}  & 26.47 & 94.48 & 52.32 & 88.90\\
GODIN~\cite{hsu2020generalized}  & 17.42  & 95.84 & 60.69 & 82.37 \\
Energy score~\cite{liu2020energy}  & 28.40 & 94.22 & 50.64 & 88.66 \\
ReAct~\cite{sun2021react} & 33.12 & 94.32 & 53.51 & 88.96\\
GradNorm~\cite{huang2021importance} & 24.79 & 92.58 & 65.44 & 79.31\\
LogitNorm~\cite{wei2022mitigating}  & 19.61 & 95.51 & 55.08 & 88.03\\
DICE~\cite{sun2022dice}  & 20.83 & 95.24 & 58.60 & 87.11 \\
\midrule
\emph{Using feature representations}\\
Mahalanobis~\cite{lee2018simple} & 44.55 & 82.56 & 87.71 & 78.93 \\
KNN~\cite{sun2022knn}  & 18.50 & 96.36 & 58.34 & 87.90 \\
\midrule 
 \name (ours) & \textbf{14.99} & \textbf{97.15}  & \textbf{50.10} &  \textbf{89.80}\\
 & $\pm{0.87}$ & $\pm{0.27}$ & $\pm{1.09}$ & $\pm{0.65}$\\
\bottomrule
\end{tabular}}
\caption{\small Performance comparison on near-OOD and far-OOD detection task. Architecture used is DenseNet-101 and ID data is CIFAR-10. We report the mean and variance across 3 training runs.}
\label{tab:hard_ood}
\end{table}
\begin{table}[t]
\small
\centering
\resizebox{0.99\linewidth}{!}{
\begin{tabular}{lccc}
\textbf{Method} & \textbf{FPR95}  & \textbf{AUROC} & \textbf{ID Acc.}\\
& $\downarrow$ & $\uparrow$ & $\uparrow$ \\
\toprule
\emph{Methods using model outputs}\\
MSP~\cite{hendrycks2016baseline} & 77.59 & 76.47 &  75.14\\
ODIN~\cite{liang2018enhancing} & 56.39 & 86.02 & 75.14\\
GODIN~\cite{hsu2020generalized} & 44.08 &  89.05 & 74.37\\
Energy score~\cite{liu2020energy} & 57.07 &  84.83 &  75.14\\
ReAct~\cite{sun2021react} & 75.06 & 79.51 & 66.56\\
GradNorm~\cite{huang2021importance} & 63.05 & 79.80 & 75.14\\
LogitNorm~\cite{wei2022mitigating} & 61.10 & 84.72 & 75.42\\
DICE~\cite{sun2022dice} & 49.72 & 87.23 & 68.65 \\
\midrule 
\emph{Methods using feature representations}\\
Mahalanobis~\cite{lee2018simple} & 56.93 & 80.27 &  75.14\\
KNN~\cite{sun2022knn} & 47.21 & 85.27 & 75.14\\
\midrule 
 \name (ours) & \textbf{31.25} & \textbf{90.76} & \textbf{75.59}\\
& $\pm{1.25}$ & $\pm{0.36}$ & $\pm{0.08}$\\
\bottomrule
\end{tabular}}
\caption{\small Performance comparison on CIFAR-100 dataset. We use DenseNet-101 for all baselines. Best  results are in \textbf{bold}. We report the mean and variance across 3 different training runs.}
\label{tab:cifar-100}
\end{table}
In this section, we extensively evaluate the effectiveness of our proposed method. 
The goal of our experimental sections is to mainly answer the following questions: (1) Can \name alleviate the curse of dimensionality? (2) How does \name compare against the state-of-the-art OOD detection methods?  Due to space constraints, extensive experimental details are in Appendix C. Our code is open-sourced for the research community.

\subsection{Evaluation on Common Benchmarks}
\label{subsec:common_benchmark}

\noindent \textbf{Datasets.} In this section, we make use of commonly studied CIFAR-10 (10 classes) and CIFAR-100 (100 classes)~\cite{krizhevsky2009learning} datasets as ID. Both datasets consist of images of size $32 \times 32$. We use the standard split with $50,000$ images for training and $10,000$ images for testing. We evaluate the methods on common OOD datasets: \texttt{Textures}~\cite{cimpoi2014describing}, \texttt{SVHN}~\cite{svhn}, \texttt{LSUN-Crop}~\cite{yu2015lsun}, \texttt{LSUN-Resize}~\cite{yu2015lsun}, \texttt{iSUN}~\cite{xu2015turkergaze}, and \texttt{Places365}~\cite{zhou2017places}. Images in all these test datasets are of size $32 \times 32$.

\paragraph{Evaluation metrics.} We compare the performance of various methods using the following metrics: 
(1) {FPR95} measures the false positive rate (FPR) of OOD samples when $95\%$ of ID samples are correctly classified;
(2) {AUROC} is the area under the Receiver Operating Characteristic curve; 
and (3) {ID Acc.} measures the ID classification accuracy.

\vspace{0.2cm}
\noindent \textbf{Comparison with competitive methods.} In Table~\ref{tab:cifar-100}, we provide a comprehensive comparison with competitive OOD detection baselines on  CIFAR-100. {We provide a detailed description of baseline approaches in Appendix C.3.} We observe that our proposed method \name significantly outperforms the latest rivals. For a fair comparison, we divide the baselines into two categories: methods using model outputs and methods using feature representations.
From Table~\ref{tab:cifar-100}, we highlight two salient observations: (1) Considering methods based on feature representations, \name outperforms KNN (non-parametric) and Mahalanobis (parametric) by \textbf{15.96\%} and \textbf{25.68\%} respectively in terms on FPR95. The results validate that learning feature subspace effectively alleviates the ``curse-of-dimensionality" problem that is troubling the existing KNN approach. (2) Further, \name also performs better than output-based methods such as ReAct~\cite{sun2021react}. Specifically, with CIFAR-100 as ID, \name provides a $\mathbf{43.81}\%$ improvement in FPR95 as compared to ReAct~\cite{sun2021react}. Notably, \name provides a \textbf{18.47\%} improvement compared to~\cite{sun2022dice}, a post-hoc sparsification method. While DICE can severely affect the ID test accuracy (68.65\%), \name exhibits stronger classification performance (75.59\%) by baking in the inductive bias of subspaces through training. An extensive discussion is provided in Section~\ref{sec:discussion}.

\paragraph{Evaluation on near-OOD data.} In Table~\ref{tab:hard_ood}, we compare the performance in detecting near-OOD data, which refers to samples near the ID data. Near-OOD is particularly challenging to detect, and can often be misclassified as ID. We report the performance on CIFAR-10 (ID) vs. CIFAR-100 (OOD), which is the most commonly used dataset pair for this task. We observe that \name consistently outperforms existing algorithms for near-OOD detection tasks, further demonstrating its strengths. Compared to KNN, \name reduces the FPR95 by 8.24\%. For completeness, we also provide far-OOD evaluation results on CIFAR-10, where \name achieves an average FPR95 of 14.99\%. Full result on each test dataset for CIFAR-10 is available in Appendix D.4.

\begin{table}[t]
\small
\centering
\resizebox{0.95\linewidth}{!}{
\begin{tabular}{lccc}
\textbf{Method} & \textbf{Dataset (ID)} & \textbf{FPR95}  & \textbf{AUROC} \\
& & $\downarrow$ & $\uparrow$  \\
\toprule
Mahalanobis~\cite{lee2018simple} & CIFAR-10 & 44.55 & 82.56  \\

\name (w. Mahalanobis) & CIFAR-10 &  \textbf{34.68} &  \textbf{87.87} \\
\midrule
Mahalanobis~\cite{lee2018simple} & CIFAR-100 & 56.93 & 80.27  \\

\name (w. Mahalanobis) & CIFAR-100 &  \textbf{55.05} &  \textbf{80.77} \\
\bottomrule
\end{tabular}}
\caption{\small \name is also compatible with parametric approaches such as Mahalanobis distance~\cite{lee2018simple}. The model is DenseNet. All values are averaged over six OOD test datasets.}
\label{tab:compatibility}
\end{table}

\paragraph{Compatibility with other distance-based approaches.} 

Beyond KNN~\cite{sun2022knn}, the Mahalanobis distances~\cite{lee2018simple} is also one of the most popular distance-based approaches to detect OOD. 
However, all prior solutions measure the distance with a full feature space which can also suffer from the curse of dimensionality. 
In this section, we show that subspace learning can also benefit parametric approaches like Mahalanobis distance~\cite{lee2018simple}. In Table~\ref{tab:compatibility}, we compare the OOD detection performance of using Mahalanobis distance on the vanilla model and the model trained with \name. 
We see that coupling subspace learning (in training) with Mahalanobis distance (in testing) reduces FPR95 by {9.87\%} and {1.88\%} on CIFAR-10 and CIFAR-100 datasets respectively.

\begin{table}[t]
\small
\centering
\resizebox{0.55\linewidth}{!}{
\begin{tabular}{lcc}
\toprule
\multirow{2}{2cm}{\textbf{Training Method}} &  CIFAR-10 & CIFAR-100 \\ 
& \multicolumn{2}{c}{(Train time in hours)} \\
\midrule
Standard & $2.10$ & $2.25$\\ 
 \name & $1.75$ & $1.89$ \\
\bottomrule
\end{tabular}}
\caption{\small \textbf{Computational cost for training}. trained using ResNet-101. 
Model used is DenseNet-101. For the comparison, we used the software configuration as reported in Appendix C.2.}
\label{tab:train_time}
\end{table}
\paragraph{Computational complexity.}  In Table~\ref{tab:train_time}, we compare the training time of \name with the standard training method using cross-entropy loss. We observe that training using \name incurs no additional computation overhead but rather is slightly more efficient compared to standard training procedures. This is because we perform gradient descent only on a subset of weights corresponding to the selected feature subspace. Thus, our method overall leads to faster updates and convergence. {In Appendix D.1, we further show that
\name remains competitive and outperforms the KNN counterpart on other common architecture.}

\section{Further Understanding of \name}
\label{sec:ablations}

Through comprehensive evaluations in Section~\ref{sec:experiment}, we have established the effectiveness of \name for OOD detection. 
\noindent In this section, we provide an in-depth analysis of several questions: (1) How does the subspace dimension impact the performance? (2) How does OOD detection performance change if we change the subspace selection strategy? 
(3) How does the $k$ in $k$-NN distance impact the OOD detection performance? We show comprehensive ablation studies to answer these questions. 
 For consistency, all ablations are based on CIFAR-100 as ID dataset and DenseNet~\cite{huang2018densely} architecture unless specified otherwise.

\begin{figure*}[t!]
\centering
\begin{minipage}{.49\textwidth}
  \centering
 \vspace{0.7cm}
  \includegraphics[width = 1\columnwidth]{images/Frame_6.pdf}
  \vspace{0cm}
    \caption{\small \textbf{Left}: Effect of varying the relevance ratio ($r$) on OOD detection performance when $k=20$. \textbf{Right}: Effect of varying the number of neighbors ($k$) when $r=0.25$. }

    \label{fig:ablations}
\end{minipage}
\hspace{0.1cm}
\begin{minipage}{.48\textwidth}
  \centering
\centering
    \includegraphics[width = 0.78\columnwidth]{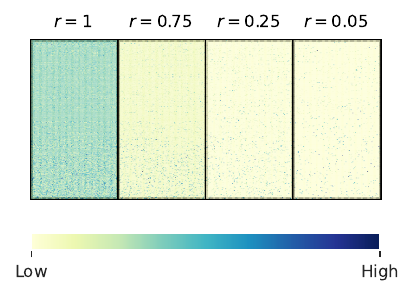}
    \caption{
    \small Visualization of learned final-layer weight matrix for $r \in \{1, 0.75, 0.25, 0.05\}$. For each $r$, we visualize a $342$-dimensional weight vector corresponding to each class in CIFAR-100. }
    \label{fig:subspace}
\end{minipage}%

\end{figure*}
\paragraph{Ablation on subspace dimension.} In this ablation, we aim to empirically verify our theoretical analysis in Section~\ref{sec:theory}, and understand the effect of subspace dimension. 
We start by defining the \textbf{relevance ratio} $r = \frac{s}{m} \in (0,1]$, which captures the sparsity of feature space used in training \name. Recall that $m$ is the original dimension of features and $s$ is the dimension we kept in \name. 

In simple words, $r$ represents the ratio between the dimension of the subspace and the full feature dimension.
Specifically, we train and compare multiple models by varying $r = \{ 0.05, 0.15, 0.25, 0.35, 0.55, 0.75\}$. Figure~\ref{fig:ablations} (left) summarizes the effect of $r$ on OOD detection. We observe that: (1) Irrespective of the relevance ratio used, \name is consistently better than KNN~\cite{sun2022knn} when $r < 1$. This validates the efficacy of learning feature subspaces for OOD detection, without excessive hyperparameter tuning. (2) Setting a mild ratio (\emph{e.g.} $r=0.25$) provides the optimal OOD detection performance, which is consistent with one chosen using our validation strategy (see Appendix C.4). (3) In the extreme case, when $r$ is too small (\emph{e.g.} $r={0.05}$), we observe a deterioration in the OOD detection performance. Overall our empirical observations align well with our theoretical insight provided in Section~\ref{sec:theory}. 

\paragraph{Visualization of the learned weight matrix.}

To further verify our method, we visualize in Figure~\ref{fig:subspace} the 
learned final-layer weight matrix under different relevance ratios ($r = s/m$). For each $r$, we visualize the $342$ dimensional weight vector corresponding to each class in CIFAR-100, \emph{i.e.}, a $342 \times 100$ matrix. When $r=1$ (\emph{i.e.} without any subspace constraint), the model utilizes the full feature space. Further, the visualization confirms that decreasing the relevance ratio effectively reduces feature subspace dimensionality. 

\begin{table}[h]
\centering
\small
\resizebox{0.85\linewidth}{!}{
\begin{tabular}{lccc}
\textbf{Method} & \textbf{FPR95}  & \textbf{AUROC} \\
& $\downarrow $& $\uparrow$ \\
\toprule
Random subspace~\cite{ho1998nearest} & 42.21 & 84.97 \\
Subspace with least relevance & 63.96 & 80.82\\
\name (ours) & \textbf{31.25} & \textbf{90.76} \\
\bottomrule
\end{tabular}}
\caption{\small Ablation on subspace selection methods. Best performing results are marked in \textbf{bold}. Model is DenseNet. All values are averaged over multiple OOD test datasets.}
\label{tab:subspace-selection}
\end{table}

\paragraph{Ablation on subspace selection.} A core component of our algorithm involves selecting the  \emph{most relevant} dimension for a class prediction. In particular, the subspace is chosen based on the dimensions that contributed most to the model's output. In this ablation, we contrast our subspace selection mechanism with random subspace~\cite{ho1998nearest}, a classical alternative. The random subspace relies on a stochastic process that randomly selects $s$ components in the feature vector. Simply put, this approach randomly sets $s$ out of $m$ elements in each $R_c$ vector to be 1 and 0 elsewhere. We report ablation results  in Table~\ref{tab:subspace-selection}. For a fair comparison, we use relevance ratio $r = 0.25$ and the number of neighbors $k = 20$ for all methods.  Empirical results highlight that randomly chosen subsets of dimensions are sub-optimal for the OOD detection task. Lastly, we also contrast with selecting the \emph{least relevant} feature dimensions. We replace Equation~\ref{eq:relevance} with: 
\begin{align}
     f_c(\*x; \theta, R_c(\*x)) & =  \min_{R_c(\*x)\in \{0,1\}^m, \|R_c(\*x)\|_1 = s} \langle \*w_c, h(\*x) \odot R_c(\*x) \rangle,
     \label{eq:least_relevance}
\end{align}
which essentially changes from \textsc{max} to \textsc{min}.  As expected, using the least relevant feature dimension results in significantly worse OOD detection performance.

\paragraph{Ablation on number of nearest neighbors $k$.} 
In Figure~\ref{fig:ablations} (right), we visualize the effect of varying the number of nearest-neighbors ($k$) on OOD detection performance. Here the model is trained with $r=0.25$ on CIFAR-100. Specifically, we vary $k \in \{5,10,20,50,100,200,500,1000\}$. We observe that the  OOD detection performance is relatively stable under a mild $k$. In Appendix D.2, we further visualize the interaction between the two hyper-parameters $r$ and $k$ through OOD detection performance. 

\vspace{0.1cm}
\noindent\textbf{\name improves calibration performance.} {
In addition to superior OOD detection performance, we aim to further  investigate the calibration performance of ID data itself. As a quick recap,  the calibration performance measures the alignment between the model's confidence and its actual predictive accuracy. We hypothesize that learning the feature subspace helps alleviate the problem of over-confident predictions on ID data, thereby improving model calibration. We verify in  Appendix D.3 that training with \name indeed significantly improves the model calibration. }

\vspace{0.1cm}
\noindent\textbf{\name scales to large datasets.} 
In this section, we evaluate \name on a more realistic high-resolution dataset ImageNet ~\cite{deng2009imagenet}. Compared to CIFAR-100, inputs scale up in size in ImageNet-100 (we follow standard data augmentation pipelines and resize the input to 224 by 224). For OOD test datasets, we use the same ones in~\cite{huang2021mos}, including subsets of \texttt{Places365}, \texttt{Textures}, \texttt{iNaturalist} and \texttt{SUN}. 
We train the ResNet-101 model for 100 epochs using a batch size of 256, starting from randomly initialized weights. We use SGD with a momentum of $0.9$, and a weight decay of 1e-4. We set the initial learning rate as $0.1$ and use a cosine-decay schedule. We set $r = 0.35$ and $k = 200$ based on our validation strategy described in Appendix C.4. We contrast two models trained with and without subspace learning. The results in terms of FPR95 and AUROC are shown in Figure~\ref{fig:imagenet}. The results suggest that \name remains effective on all the OOD test sets and consistently outperforms KNN. This further verifies the benefits of explicitly promoting feature subspace to combat curse-of-dimensionality. 

\begin{figure}[t]
    \centering
    \includegraphics[width = \linewidth]{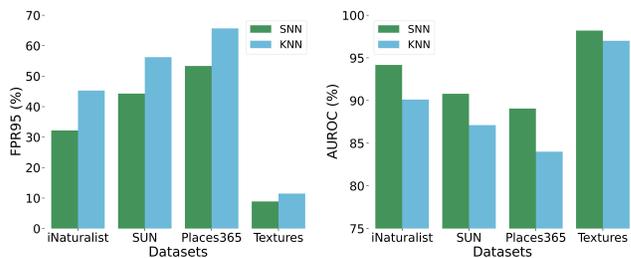}
    \caption{\small OOD detection performance comparison on ImageNet dataset (ID). \name consistently outperforms KNN across all OOD test datasets on the same architecture.  }
   
    \label{fig:imagenet}
\end{figure}

\section{Discussion}
\label{sec:discussion}

\paragraph{Relations to DICE~\cite{sun2022dice}.} Our work differs from DICE in two crucial
aspects, both in terms of training loss and test-time OOD detection mechanism.
As introduced in Section~\ref{sec:method}, \name can be viewed as \emph{training-time} regularization with subspace. In contrast, DICE~\cite{sun2022dice}  only explored \emph{test-time} weight sparsification mechanism, without explicitly learning the subspace in training. Specifically, DICE employs the standard cross-entropy loss, whereas we contribute a new learning objective (\emph{c.f.} Section~\ref{sec:train}) for OOD detection. Importantly, we show that training-time regularization provides substantial gains over simple post-hoc sparsification. For example, as evidenced in Table~\ref{tab:cifar-100},  post-hoc sparsification can severely affect the ID test accuracy. In contrast, our method bakes the inductive bias of ``feature subspace'' into training time, and exhibits stronger ID generalization in testing time. Besides, another major difference between \name and DICE lies in how test-time OOD detection is performed. In particular, \name is a distance-based method that operates in the \emph{feature} space, whilst DICE derives OOD scores from the model \emph{output} space. Different from DICE, \name is motivated to address the curse-of-dimensionality issue for OOD detection. Table~\ref{tab:cifar-100} compares the OOD detection performance of \name with DICE. For example, on CIFAR-100, \name reduces FPR95 by \textbf{18.47\%} compared to DICE. Improved performance validates the advantage of learning a subspace in training as opposed to test-time sparsification.

\paragraph{Relations to Dropout strategies.} A commonly used technique to prevent over-parameterization is Unit Dropout~\cite{hinton2012improving,srivastava2014dropout}. Unlike Random Subspace~\cite{ho1998nearest}, Unit Dropout randomly sets elements in the feature activation vector $h(\*x)$ to be 0. \citet{ba2013adaptive} propose Adaptive Dropout, which generalizes the prior approach by learning the dropout probability using a binary belief network. \citet{gomez2019learning} proposed Targeted Dropout which is based on keeping top-$k$ weights in a fully connected layer with the highest magnitude. Their motivation is that low-valued weights should be allowed to increase their value during training, if they become important. Hence, instead of completely pruning the rest of the weights, they apply dropout with a fixed rate to randomly drop a fraction of the low-valued weights. We contrast the OOD detection performance of \name with these dropout strategies in Table 7 (Appendix). Notably, \name improves FPR95 by \textbf{37.90\%} and \textbf{20.14\%} as compared to Targeted Dropout~\cite{gomez2019learning} and Adaptive Dropout~\cite{ba2013adaptive}, respectively. Lastly, \citet{wong2021leveraging} proposed an elastic net formulation to enforce sparsity for model interpretability. In Appendix D.5, we provide an extended discussion contrasting the OOD detection performance of \name with \cite{wong2021leveraging}.

\section{Related Works}
\label{sec:related works}

\vspace{0.1cm}
\noindent\textbf{Training-based OOD detection methods.} Commonly used approaches in this direction design training-time loss functions and regularization for better ID/OOD separability. Initial work by \cite{devries2017improved} proposed to augment the network with a confidence estimation branch. MOS~\cite{huang2021mos} improves OOD detection by modifying the loss to use a pre-defined group structure. Another branch of studies based on train-time regularization have also shown promising and significant improvement in OOD detection performance~\cite{lee2017training, bevandic2018discriminative, hendrycks2018deep, geifman2019selective, hein2019why, meinke2019towards, liu2020energy, wei2022mitigating,  KatzSamuels2022, Du2022, Du2022a, ming2022posterior, hebbalaguppenovel, ming2023how, huang2023harnessing, wang2023outofdistribution, bai2023feed, wang2023learning, du2023dream, zheng2023detection}. Common approaches include training models to give predictions with lower confidences~\cite{lee2018simple, hendrycks2018deep, wang2022outofdistribution} or higher energies~\cite{liu2020energy, Du2022, song2022rankfeat} for outlier samples. 
    Most regularization methods assume the availability of a large-scale and diverse outlier dataset which is not always realizable. 
    Different from these methods, our proposed method does not require any additional outlier data. Rather, in this work, we formulate a novel train-time regularization approach based on learning feature subspaces. Our empirical analysis highlights the benefits of subspace learning for OOD detection.

\vspace{0.1cm}
 \noindent\textbf{Inference-based OOD detection methods.} Studies in this domain mainly focus on designing scoring functions for OOD detection. These approaches include for example: (1) confidence-based methods~\cite{hendrycks2016baseline, liang2018enhancing}, (2) energy-based methods~\cite{liu2020energy, morteza2022provable}, (3) gradient-based method~\cite{huang2021importance}, and (4) {feature-based methods}~\cite{Wilson_2023_ICCV}. Some representative works include Mahalanobis distance~\cite{lee2018simple, 2021ssd} and non-parametric KNN distance~\cite{sun2022knn}. However, the efficacy of these metrics is often limited in higher dimensions due to the ``curse-of-dimensionality'' (\textit{c.f.} Section~\ref{sec:theory}). Our paper targets precisely this critical yet underexplored problem. 
    We show new insights that learning feature subspaces effectively alleviate this problem. 

\vspace{0.1cm}
\noindent\textbf{Subspace learning.} To overcome limitations of ``curse-of-dimensionality" in high dimensions, \citet{ho1998nearest} explored the idea of random subspace
for nearest neighbor. Subspace learning has also been used to improve search queries in high dimensions~\cite{hund2015subspace}, feature selection~\cite{liu2007computational}, finding clusters in arbitrarily oriented spaces~\cite{kriegel2009clustering}, and intrinsic dimensionality estimation~\cite{houle2014efficient}.  Different from these previous works, we explore class-relevant feature subspace learning for OOD detection.

\section{Conclusion}
\label{sec:conclusion}

Our work highlights the challenge of curse-of-dimensionality in OOD detection, and introduces a new solution called \name for detecting OOD samples. Traditional distance-based methods for OOD detection suffer from the curse-of-dimensionality, which makes it difficult to distinguish between ID and OOD samples in high-dimensional feature spaces. To address this issue, \name learns subspaces that capture the most informative feature dimensions for the task. Our method is supported by theoretical analysis, which shows that reducing the feature dimensions improves the distinguishability between ID and OOD samples. Extensive experiments demonstrate that \name achieves significant improvements in both OOD detection and ID calibration performance. We hope that our approach will inspire future research on this important problem.

\section*{Acknowledgement}
Research is supported by the AFOSR Young Investigator Program under award number FA9550-23-1-0184, National Science Foundation (NSF) Award No. IIS-2237037 \& IIS-2331669, Office of Naval Research under grant number N00014-23-1-2643, and faculty research awards/gifts from Google and Meta.  Any opinions, findings, conclusions, or recommendations
 expressed in this material are those of the authors and do not necessarily reflect the views, policies, or endorsements either expressed or implied, of the sponsors.

\bibliography{egbib}
\clearpage

\appendix
\section{Societal Impact}
\label{sec:impact}
In this paper, we show that in high-dimensional spaces, the efficacy of distance-based out-of-distribution (OOD) detection methods can be limited by curse-of-dimensionality. To combat this problem, we propose a novel framework of subspace learning for OOD detection. OOD detection is a critically important component for a vast range of systems which include
business applications (e.g., content understanding), transportation (e.g., autonomous vehicles), and health care (e.g., unseen disease identification). Our study has positive societal impacts. We hope that it will further enhance the understanding regarding the crucial issue of how curse-of-dimensionality affects distance-based OOD detection methods. Our study does not involve any human subjects or violation of legal compliance. We do not anticipate the potentially harmful consequences of our work. Through our study and releasing our code, we hope to raise stronger research and societal attention to the problem of OOD detection.

\section{Proof of Main Theorem}
\label{app:proof}

\begin{theorem} (Recap of Th.~\ref{th:main}) We let $\mathbb{E}[p_{in}(\*z)|{\*z \in \mathcal{Z}_{in}}] - \mathbb{E}[p_{in}(\*z)|{\*z \in \mathcal{Z}_{out}}] = \Delta(m)$ as a function of the feature's dimensionality $m$. We have the following bound:
    \begin{align}
            \hat{\Delta}(m) & \geq \Delta(m) - O((\frac{k}{N})^{\frac{1}{m}} + k^{-\frac{1}{2}})
    \end{align}
\label{th:sup_main}
\end{theorem}

\begin{proof}

\begin{align*}
     &\mathbb{E}[p_{in}(\*z)|{\*z \in \mathcal{Z}_{in}} ]  - 
    \mathbb{E}[p_{in}(\*z)|{\*z \in \mathcal{Z}_{out}}]
    \\ 
    &~~=  \mathbb{E}[p_{in}(\*z) - \hat{p}_{in}(\*z)|{\*z \in \mathcal{Z}_{in}}] + 
    \mathbb{E}[\hat{p}_{in}(\*z)|{\*z \in \mathcal{Z}_{in}}] \\
    &~- \mathbb{E}[\hat{p}_{in}(\*z)|{\*z \in \mathcal{Z}_{out}}] +
    \mathbb{E}[\hat{p}_{in}(\*z) - p_{in}(\*z)|{\*z \in \mathcal{Z}_{out}}] 
    \\ &~~\leq  \mathbb{E}[|p_{in}(\*z) - \hat{p}_{in}(\*z)| | {\*z \in \mathcal{Z}_{in}}] + 
    \mathbb{E}[\hat{p}_{in}(\*z)|{\*z \in \mathcal{Z}_{in}}]\\
    &~- \mathbb{E}[\hat{p}_{in}(\*z)|{\*z \in \mathcal{Z}_{out}}] +
    \mathbb{E}[|\hat{p}_{in}(\*z) - p_{in}(\*z)| | {\*z \in \mathcal{Z}_{out}}] 
    \\
    &~~= \frac{\int_{\mathcal{Z}_{in}} |p_{in}(\*z) - \hat{p}_{in}(\*z)| p_{in}(\*z) d\*z }{\int_{\mathcal{Z}_{in}}  p_{in}(\*z) d\*z } \\
    &~+ \frac{\int_{\mathcal{Z}_{out}} |p_{in}(\*z) - \hat{p}_{in}(\*z)| p_{in}(\*z) d\*z }{\int_{\mathcal{Z}_{out}}  p_{in}(\*z) d\*z }  \\
    &~+  \mathbb{E}[\hat{p}_{in}(\*z)|{\*z \in \mathcal{Z}_{in}}] - \mathbb{E}[\hat{p}_{in}(\*z)|{\*z \in \mathcal{Z}_{out}}]\\
    \\ &\leq \frac{\int_{\mathcal{Z}_{in}} |p_{in}(\*z) - \hat{p}_{in}(\*z)| p_{in}(\*z) d\*z  + \int_{\mathcal{Z}_{out}} |p_{in}(\*z) - \hat{p}_{in}(\*z)| p_{in}(\*z) d\*z }{\min(\int_{\mathcal{Z}_{in}}  p_{in}(\*z) d\*z, \int_{\mathcal{Z}_{out}}  p_{in}(\*z) d\*z) }\\
    &~+ \mathbb{E}[\hat{p}_{in}(\*z)|{\*z \in \mathcal{Z}_{in}}] - \mathbb{E}[\hat{p}_{in}(\*z)|{\*z \in \mathcal{Z}_{out}}] \\
    \\ &~~= \frac{\int_{\mathcal{Z}} |p_{in}(\*z) - \hat{p}_{in}(\*z)| p_{in}(\*z) d\*z  }{\min(\int_{\mathcal{Z}_{in}}  p_{in}(\*z) d\*z, \int_{\mathcal{Z}_{out}}  p_{in}(\*z) d\*z) } \\
    &~+ \mathbb{E}[\hat{p}_{in}(\*z)|{\*z \in \mathcal{Z}_{in}}] - \mathbb{E}[\hat{p}_{in}(\*z)|{\*z \in \mathcal{Z}_{out}}]
    \\ &~~= \frac{\mathbb{E}[|p_{in}(\*z) - \hat{p}_{in}(\*z)|] }{\min(\int_{\mathcal{Z}_{in}}  p_{in}(\*z) d\*z, \int_{\mathcal{Z}_{out}}  p_{in}(\*z) d\*z) } \\
    &~+ \mathbb{E}[\hat{p}_{in}(\*z)|{\*z \in \mathcal{Z}_{in}}] - \mathbb{E}[\hat{p}_{in}(\*z)|{\*z \in \mathcal{Z}_{out}}].
\end{align*}

\begin{lemma}
According to Theorem 1 in ~\cite{zhao2022analysis}, the estimation error of $k$-NN distances can be bounded by:
$$
\mathbb{E}[|p_{in}(\*z) - \hat{p}_{in}(\*z)|]  \leq \mathcal{O}(\left(\frac{k}{N}\right)^{\frac{1}{m}} + k^{-\frac{1}{2}})
$$
\label{lemma:est_err}
\end{lemma}
By Lemma~\ref{lemma:est_err}, we have the final results: 
\begin{align*}
     \mathbb{E}[\hat{p}_{in}(\*z)|{\*z \in \mathcal{Z}_{in}}] - & \mathbb{E}[\hat{p}_{in}(\*z)|{\*z \in \mathcal{Z}_{out}}] \geq \mathbb{E}[p_{in}(\*z)|{\*z \in \mathcal{Z}_{in}} ] \\
     &- \mathbb{E}[p_{in}(\*z)|{\*z \in \mathcal{Z}_{out}}] - O((\frac{k}{N})^{\frac{1}{m}} + k^{-\frac{1}{2}})
\end{align*}

\end{proof}

\section{Supplementary Experiment Details}
\label{app:experimental_details}
 \subsection{Training details} 
\label{app:train_details}
For main experimentation, we train DenseNet-101~\cite{huang2018densely} for 100 epochs using SGD with a momentum of 0.9, a weight decay of 0.0005, and a batch size of 64. We set the
initial learning rate as 0.1 and reduce it by a factor of 10 at 50, 75, and 90 epochs. For ResNet-50~\cite{he2016deep}, we use SGD with a momentum of 0.9, weight decay of 0.0001, batch size of 128, and train the model for 100 epochs. The learning rate is adjusted using the same schedule as used for training the DenseNet model. 
The relevance ratio $r \in \{0.05, 0.15, 0.25, 0.35, 0.55, 0.75\}$ and number of neighbors $k \in \{5,10,20,50,100,200,500,1000\}$ are cross-validated as described in Appendix~\ref{app:validation}. For all experiments on CIFAR-10/100 benchmark using DenseNet~\cite{huang2018densely}, we use $r=0.25$ and $k=20$ based on our validation strategy. For experimentation using ResNet-50~\cite{he2016deep}, we set $r=0.05$ and $k=20$. We report ablation results for the effect of $r$ and $k$ in Section~\ref{sec:ablations}.

\subsection{Software and Hardware}
\label{app:hardware}
We run all experiments with Python 3.7.4 and PyTorch 1.9.0. For all experimentation, we use Nvidia RTX 2080-Ti and A6000 GPUs.

\subsection{Description of OOD baselines}
\label{app:ood_description}
In this section, we include a brief description of all the OOD baseline methods.
\subsubsection{Methods using model outputs}

\paragraph{Maximum Softmax Probability (MSP)~\cite{hendrycks2016baseline}} uses the maximum softmax probability (or the confidence score) to detect OOD examples.

\paragraph{ODIN~\cite{liang2018enhancing}} ODIN utilizes the confidence score after temperature scaling and input perturbations for OOD detection. We set temperature parameter $T=1000$ for all experiments on the CIFAR-10/100 benchmark. Perturbation Magnitude $\eta$ is chosen by validating on 1000 images randomly sampled from the ID test set. We set the perturbation magnitude $\eta = 0.0016$ for CIFAR-10 and $\eta = 0.0012$ for CIFAR-100.

\paragraph{Energy~\cite{liu2020energy}} Liu~\etal~proposed using energy score for OOD detection. The energy function maps the logits to a scalar output, which is relatively
lower for ID data. This score is hyperparameter free and does not require any tuning.

\paragraph{Generalized-ODIN \cite{hsu2020generalized}} Hsu \etal~propose a decomposed confidence model for the purpose of OOD detection, where the logits of a classifier are defined using a dividend/divisor structure. The authors propose three variants of OOD detectors, namely, DeConf-I, DeConf-E, and DeConf-C --- which uses Inner-Product, Negative Euclidean Distance, and Cosine Similarity respectively. In this study, we use the DeConf-C variant, since it is shown to be the most robust of all the variants. Finally, input samples are perturbed to improve OOD performance. Similar to ODIN~\cite{liang2018enhancing}, the perturbation magnitude $\epsilon$ is chosen by validating on 1000 images randomly sampled from the ID test set. We set perturbation magnitude $\epsilon = 0.02$ for both CIFAR-10/100 benchmarks.

\paragraph{ReAct~\cite{sun2021react}} ReAct is a post-hoc OOD detection approach based on activation truncation. The paper states that the optimal OOD performance is obtained with the ReAct+Energy setting. Hence, in this study, we use the energy score for OOD detection using ReAct. Following the original paper, we calculate the clipping threshold based on the $90$-th percentile of activations estimated on the ID data.

\paragraph{GradNorm~\cite{huang2021importance}} GradNorm employs the magnitude of gradient vectors for detecting OOD samples. The gradient is derived from the KL-divergence between the softmax output and uniform probability distribution. For GradNorm, following the original implementation, we set the temperature $T = 1$.

\paragraph{LogitNorm~\cite{wei2022mitigating}} LogitNorm proposes a simple fix to the common cross-entropy loss by enforcing a constant vector norm on the logits during training. A temperature parameter $\tau$ is used to modulate the magnitude of the logits. In this study, we set $\tau = 0.04$ for both CIFAR-10/100 datasets.

\paragraph{DICE~\cite{sun2022dice}} DICE ranks weights based on a measure of contribution, and selectively uses the most salient weights to derive the output for OOD detection. By pruning away irrelevant weights, DICE reduces the output variance for OOD
data, resulting in better separability between ID and OOD. Following the original implementation, we set the sparsity parameter $p = 0.9$ for both CIFAR-10/100 benchmarks.

\subsubsection{Methods using feature representations}

\paragraph{Mahalanobis \cite{lee2018simple}} This method models the feature space as a mixture of multivariate Gaussian distributions, and calculates Mahalanobis distance~\cite{mahalanobis1936generalized} for OOD detection. The basic idea is that the testing OOD samples should be relatively far away from the centroids or prototypes of ID classes. The minimum Mahalanobis distance to all class centroids is used for OOD detection. Previous works~\cite{sun2022knn, 2021ssd} have shown that for the Mahalanobis score, stronger performance is obtained using normalized penultimate feature vectors. Hence, we use normalized penultimate feature vectors for the Mahalanobis baseline.

\paragraph{KNN~\cite{sun2022knn}} Recently Sun \etal~proposed using non-parametric nearest-neighbor distance
for OOD detection. Unlike Mahalanobis~\cite{lee2018simple}, the non-parametric approach does not impose any distributional assumption about the underlying feature space, hence providing stronger
flexibility and generality. Following original implementation, we set the number of neighbors $k=50$ for CIFAR-10 and $k=200$ for CIFAR-100.

\subsection{Validation Strategy}
\label{app:validation}
For finding the optimal value of relevance ratio $r\in\{0.05,0.15,0.25,0.35,0.55,0.75\}$ and nearest-neighbors $k\in\{5,10,20,50,100,200,500,1000\}$, we use a validation set of Gaussian noise images. For generating these images, each pixel is sampled from $\mathcal{N} (0, 1)$. We do a grid search over all possible values of $r \times k$ and the configuration providing the best AUROC is chosen as optimal. Using DenseNet-101~\cite{huang2018densely}, we find that $r = 0.25$ and $k=20$ provides the optimal performance on both CIFAR-10/100 dataset. For ResNet-50~\cite{he2016deep}, $r=0.05$ and $k=20$ provides optimal performance on CIFAR-10/100 datset. For ImageNet-100, $r = 0.35$ and $k=200$ is optimal.

\subsection{Algorithm Pseudo Code}
\label{app:pseudo}
In this section, we provide the PyTorch code for implementing SNN. Specifically, we replace the final linear layer in a neural network with the \verb|SNN| layer to learn class-relevant subspaces.

{\small 
\begin{lstlisting}[language=Python]
class SNN(nn.Linear):

    def __init__(self, in_features, out_features, bias=True, r=0.25):
        super(SNN, self).__init__(in_features, out_features, bias)
        self.r = r
        self.s = int(self.r * in_features) #subspace dimension
        
    def forward(self, input):
        vote = input[:, None, :] * self.weight
        if self.bias is not None:
            out = vote.topk(self.s, 2)[0].sum(2) + self.bias
        else:
            out = vote.topk(self.s, 2)[0].sum(2)
        return out

\end{lstlisting}}
\section{Supplementary Experimental Studies}
\subsection{Performance on Different Architectures}
\label{app:diff_arch}

 In Table~\ref{tab:cifar-100} and \ref{tab:hard_ood} (main paper), we have established the superiority of our proposed algorithm on DenseNet~\cite{huang2018densely}. Going beyond, in Table~\ref{tab:arch}, we show that \name remains competitive and outperforms the KNN counterpart for other common architectures such as ResNet~\cite{he2016deep}. From Table~\ref{tab:arch}, we observe that: (1) On ResNet-50, \name reduces FPR95 by \textbf{7.04}\% compared to the KNN baseline. This highlights precisely the benefits of using feature subspace for deriving the nearest neighbor distance. In contrast, \cite{sun2022knn} employed the original feature space, where irrelevant feature dimensions can impede the separability between ID and OOD data.
(2) Learning subspace during training time can preserve the ID test accuracy for both architectures. 


\begin{table}[t]
\centering
\small 

\resizebox{0.95\linewidth}{!}{%
\begin{tabular}{lcccc}
\textbf{Method} & \textbf{Architecture} & \textbf{FPR95}  & \textbf{AUROC} & \textbf{ID Acc.}\\
& & $\downarrow$ & $\uparrow$ & $\uparrow$ \\
\toprule
KNN~\cite{sun2022knn} & DenseNet-101 & 47.21 & 85.27 & 75.14\\
\name (ours) & DenseNet-101 &  \textbf{31.25} &  \textbf{90.76} & \textbf{75.59}\\
\midrule
KNN~\cite{sun2022knn} & ResNet-50 & 53.05 & 83.61 & 74.07\\
\name (ours) & ResNet-50 &  \textbf{46.01} & \textbf{86.54} & \textbf{74.36}\\
\bottomrule
\end{tabular}}
\caption{\small Performance comparison on CIFAR-100 dataset for various network architectures. All values are averaged over multiple OOD test datasets. The best results are in \textbf{bold}.}
\label{tab:arch}

\end{table}

\subsection{Understanding relationship between $r$ and $k$}
\label{app:rel}

In Figure~\ref{fig:ablations} (main paper), we show how varying the relevance ratio ($r$) and nearest-neighbor ($k$) independently modulate the OOD detection performance. In Figure~\ref{fig:fpr} and Figure~\ref{fig:auroc}, we visualize the relationship between the hyper-parameters $r$ and $k$ through OOD detection performance. The model is DenseNet-101 and ID is CIFAR-100. We observe: (1) for all values of $r$, the OOD performance is relatively stable for a mild value of $k$. (2) $r=0.25$ provides the optimal OOD performance which is the same as obtained by our validation strategy (Appendix~\ref{app:validation}).

\begin{table}[t]
\small
\centering
\resizebox{0.99\linewidth}{!}{
\begin{tabular}{lccc}
\textbf{Method} & \textbf{FPR95}  & \textbf{AUROC} & \textbf{ID Acc.}\\
& $\downarrow$ & $\uparrow$ & $\uparrow$ \\
\toprule
Wong et al.~\cite{wong2021leveraging} & 68.06 & 79.63 & 65.89\\
Unit Dropout~\cite{srivastava2014dropout} & 62.98 & 81.40 & 72.37\\
Adaptive Dropout~\cite{ba2013adaptive} & 51.39 & 81.57 & 75.39\\
Targeted Dropout~\cite{gomez2019learning} & 69.15 & 79.80 & 73.26\\
 \name (ours) &  \textbf{31.25} & \textbf{90.76} & \textbf{75.59} \\
\bottomrule
\end{tabular}}
\caption{\small Ablation on training-time regularization methods. For OOD detection using Dropout algorithms, we calculate KNN score~\cite{sun2022knn} using feature vector $h(\*x)$. }
\label{tab:sparsification-method}
\end{table}
\begin{figure*}[h!]
\begin{subfigure}{0.5\textwidth}
  
  \includegraphics[width=0.90\textwidth]{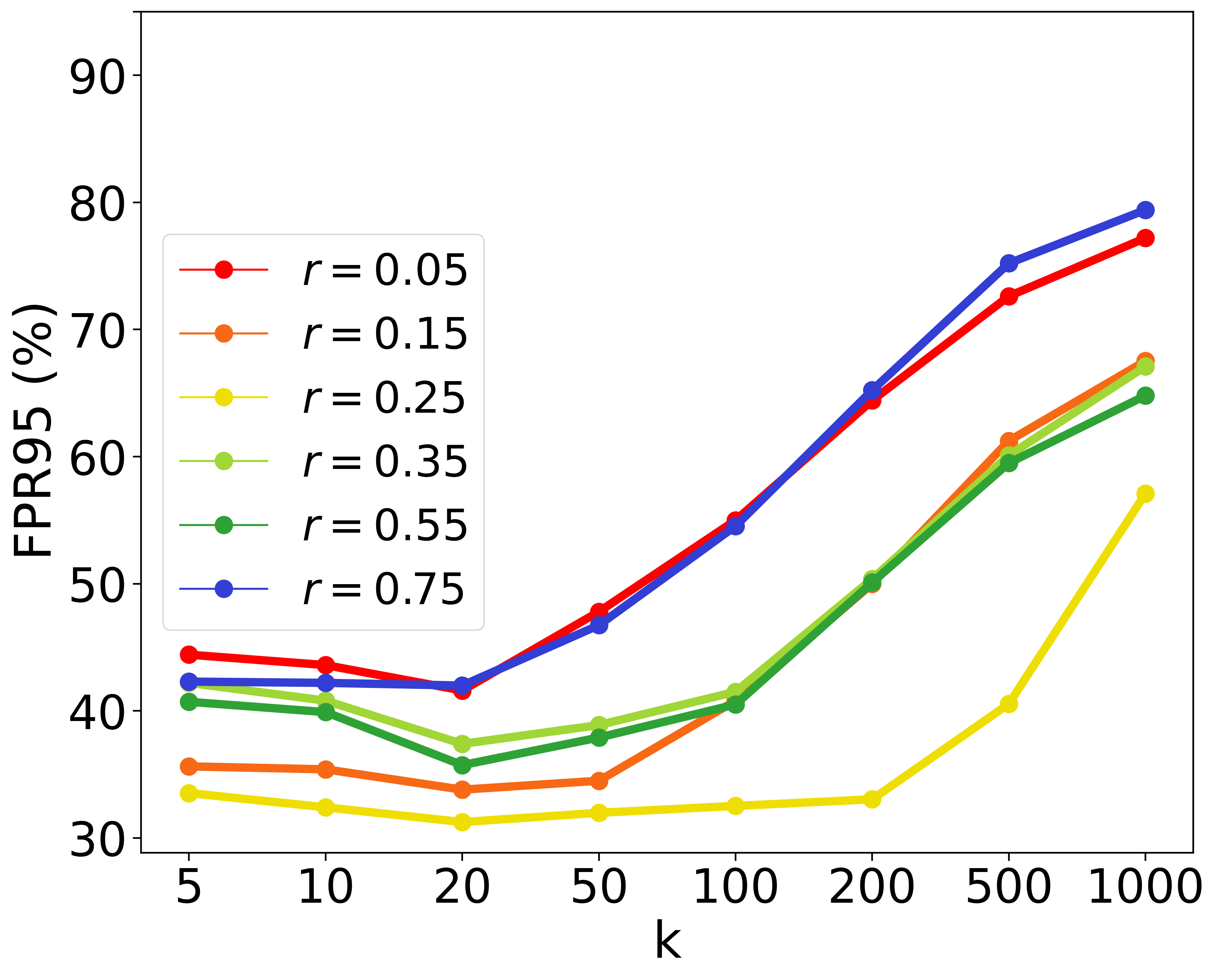}
  \caption{}
  \label{fig:fpr}
\end{subfigure}%
\begin{subfigure}{0.5\textwidth}
  \includegraphics[width=0.90\textwidth]{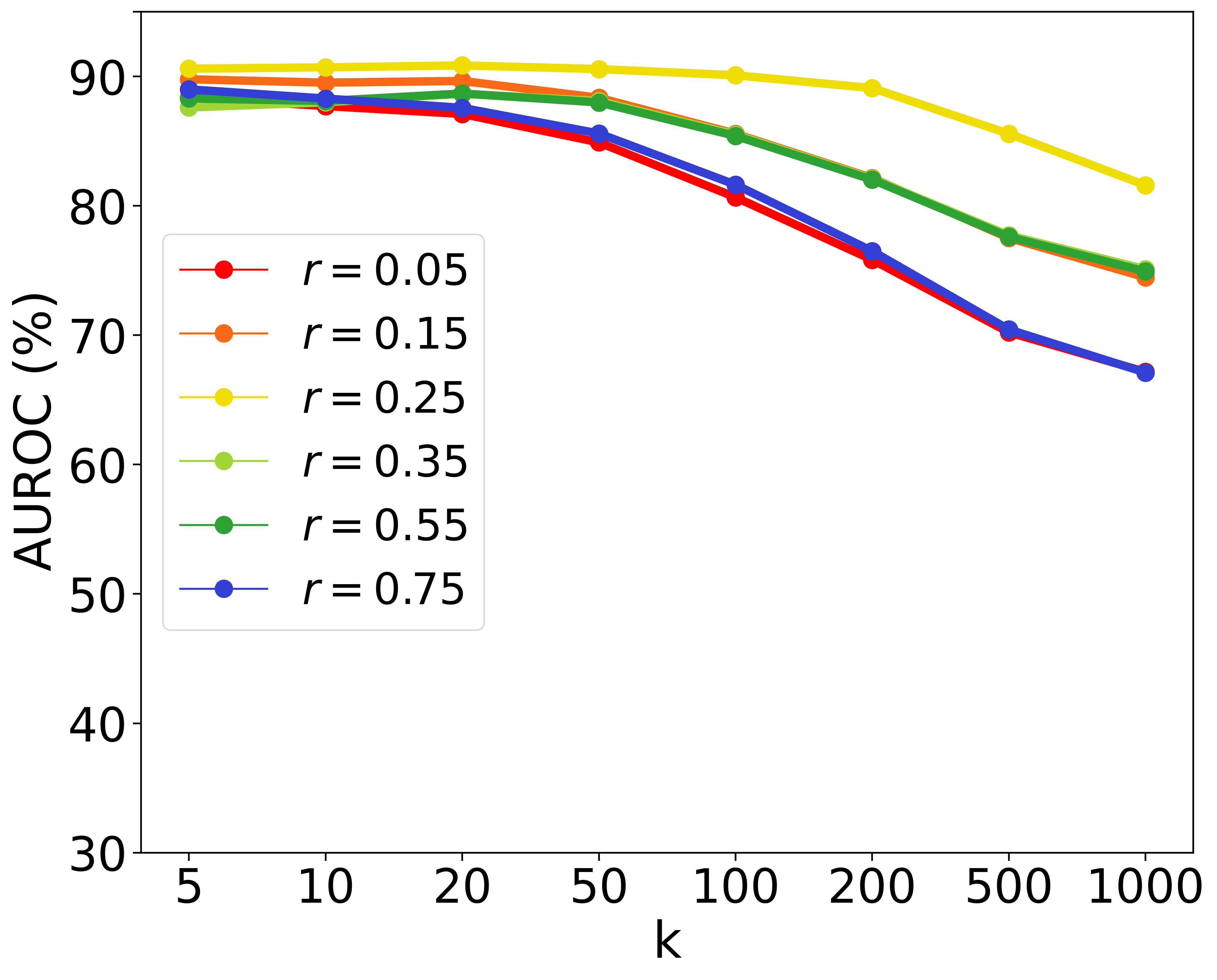}
  \caption{}
  \label{fig:auroc}
\end{subfigure}
\caption{\small Visualization of the relationship between the hyper-parameters $r$ and $k$ through OOD detection performance. The model is DenseNet-101 and ID is CIFAR-100. The OOD performance is averaged over six test datasets as mentioned in Section~\ref{sec:experiment}.} 
\label{fig:relationship}
\end{figure*}

\subsection{Evaluation on Calibration}
\label{app:calibration}

\newcolumntype{?}{!{\vrule width 1pt}}
\begin{table*}[h!]
\centering
\resizebox{0.9\textwidth}{!}{%
\begin{tabular}{lcccc?cccc?cccc}
\multirow{2}{*}{\textbf{Dataset}} & \multicolumn{2}{c}{\textbf{NLL}} & \multicolumn{2}{c?}{\textbf{NLL (w. Subspace)}} & 
\multicolumn{2}{c}{\textbf{LS}} & \multicolumn{2}{c?}{\textbf{LS (w. Subspace)}} & \multicolumn{2}{c}{\textbf{FL}} & \multicolumn{2}{c}{\textbf{FL (w. Subspace)}} \\

\cmidrule{2-13}
& SCE & ECE & SCE & ECE & SCE & ECE & SCE & ECE & SCE & ECE & SCE & ECE \\
\midrule
\multirow{1}{*}{CIFAR-10} & 11.5 & 5.6 & 9.6 &  4.4 & 8.8 & 3.9 & 8.2 & 3.8 & 6.9 & 2.3 & \textbf{3.9} & \textbf{0.8} \\
\midrule
\multirow{1}{*}{CIFAR-100} & 3.7 & 15.1 & 2.5 & 5.8  & 2.2 & 7.8 & 2.1 & \textbf{2.7} & 2.3 & 6.8 & \textbf{2.1} & 4.3 \\

 \bottomrule

\end{tabular}%
}
\caption{\small \textbf{Calibration Results.} Comparison of calibration performance when using subspace learning with commonly used loss functions (NLL/LS/FL). The model is ResNet-50. Best performing results are marked in \textbf{bold}.}
\label{tab:calibration_1}
\end{table*}

\noindent Now we move beyond OOD detection tasks and systematically investigate the calibration performance on ID data itself. With our subspace learning, the model learns the feature subspace for each class. We hypothesize that learning the feature subspace helps alleviate the problem of over-confident predictions, thereby improving model calibration. Following the literature, we evaluate calibration performance based on two common metrics: {Expected Calibration Error (ECE)}~\cite{naeini2015obtaining} and  {Static Calibration Error (SCE)}~\cite{nixon2019measuring}. 

\paragraph{Description of calibration baselines.}

Before comparing the calibration performance, we first provide a brief description of loss functions for model calibration, along with hyperparameters in training: (1)~\textbf{Label Smoothing~\cite{muller2019does}.} In Label smoothing (LS), instead of using a one-hot encoded target $y$, a soft target vector $\*q$ is defined for each sample. Specifically, $q_i = \frac{\alpha}{C-1}$ if $i \neq y$, else $q_i = 1-\alpha, ~~~\forall i \in \{1,2,...,C\}$. Here, $\alpha$ is a hyperparameter. In this study, we set $\alpha = 0.05$. (2)~\textbf{Focal Loss~\cite{lin2017focal, mukhoti2020calibrating}.} Given input $\*x$, let $\hat{p}_c = \mathbb{P}(\hat{y}=c|\*x)$ be the output softmax probability of $\*x$ belonging to class $c$. The focal loss is defined as -$(1 - \hat{p}_y)^{\gamma}\text{log}(\hat{p}_y)$, where $y$ is the ground truth label and $\gamma$ is a user-defined hyperparameter. Following the original implementation, we set $\gamma=3$ for all experiments.

\paragraph{Learning subspace improves calibration.} 
In Table~\ref{tab:calibration_1}, we observe that our subspace-regularized training algorithm improves calibration performance. In particular, we consider three losses that are commonly studied for calibration: (1) Cross Entropy Loss (NLL), (2) Label Smoothing~\cite{muller2019does}, and (3) Focal Loss (FL)~\cite{lin2017focal}.  We train the model with each of these losses and compare calibration performance with and without subspace regularization. For each dataset, we split the train set into two mutually exclusive sets: (1) $90\%$ of the train samples are used for training the model, and (2) the remaining $10\%$ of samples are used for validation. We observe from Table~\ref{tab:calibration_1} that \name improves the calibration performance for all loss functions.

\subsection{Detailed Results on All OOD Datasets}
\label{app:results}

In Table~\ref{tab:ablation_complete_c10} and Table~\ref{tab:ablation_complete_c100}, we report detailed results on six OOD test datasets when ID is CIFAR-10/100 respectively. The architecture used for all methods (including baselines) is DenseNet-101~\cite{huang2018densely}.

\subsection{Additional Discussion }
\label{app:add_discuss}
\paragraph{Relations to Wong \etal~\cite{wong2021leveraging}.} Wong \etal~\cite{wong2021leveraging} proposed an elastic net formulation to enforce sparsity for model interpretability. Hence, their motivation is fundamentally different from the problem we are trying to solve. Specifically, we learn a feature subspace for better ID-OOD separability, whereas Wong \etal improve the debuggability of neural nets. Given penultimate feature representations $h(\*x)$, \cite{wong2021leveraging} learns a sparse linear model $h(\*x)^{\top}\*w + w_0$ using the following optimization:
\begin{equation*}
\small  \min_{\*w}  \cfrac{1}{2N}||h(\*x)^{\top}\*w + w_0 -  y||^2_{2} - \lambda \left ( \cfrac{(1-\alpha)}{2}||\*w||_2^2 + \alpha||\*w||_1 \right),
\end{equation*}

where $\lambda$ and $\alpha$ are hyperparameters. In Table~\ref{tab:sparsification-method}, we compare the OOD performance between \name and Wong \etal. We make two concrete observations: (1) \name clearly outperforms \cite{wong2021leveraging} in terms of OOD detection performance. This result validates the effectiveness of our proposed subspace learning algorithm. (2) Model trained using ~\cite{wong2021leveraging} leads to suboptimal ID accuracy (65.89\%). In contrast, \name maintains the ID accuracy (75.59\%) along with improved OOD performance.

\begin{table*}[h!]

\centering

\resizebox{\textwidth}{!}{%
\begin{tabular}{l*{16}c}
 \multirow{2}{1.5cm}{\textbf{Methods}} & \multicolumn{12}{c}{\textbf{OOD Datasets}} & \multicolumn{2}{c}{\multirow{2}{*}{\centering\textbf{Average}}} & \multirow{2}{*}{\centering\textbf{ID Acc.}}\\

\cmidrule{2-13}

& \multicolumn{2}{c}{\textbf{SVHN}} & \multicolumn{2}{c}{\textbf{LSUN-c}} &
\multicolumn{2}{c}{\textbf{LSUN-r}} &
\multicolumn{2}{c}{\textbf{iSUN}} & \multicolumn{2}{c}{\textbf{Textures}} & \multicolumn{2}{c}{\textbf{Places365}} & &  \\

& FPR95 & AUROC & FPR95 & AUROC & FPR95 & AUROC &  FPR95 & AUROC & FPR95 & AUROC & FPR95 & AUROC & FPR95 & AUROC & \\
 & $\pmb{\downarrow}$ & $\pmb{\uparrow}$  & $\pmb{\downarrow}$ & $\pmb{\uparrow}$ & $\pmb{\downarrow}$ &  $\pmb{\uparrow}$ & $\pmb{\downarrow}$ & $\pmb{\uparrow}$ & $\pmb{\downarrow}$ & $\pmb{\uparrow}$ & $\pmb{\downarrow}$ & $\pmb{\uparrow}$ &
 $\pmb{\downarrow}$ & $\pmb{\uparrow}$ & $\pmb{\uparrow}$  \\
\midrule
\emph{Methods using Model Outputs}\\

MSP~\cite{hendrycks2016baseline} & 43.49 & 94.01 & 44.42 & 94.13 & 47.39 & 93.48 & 47.80 & 93.48 & 66.03 & 87.26 & 63.52 & 88.37 & 52.11 & 91.79 & 94.03\\
ODIN~\cite{liang2018enhancing} & 34.15 & 94.73 & 8.39 & 98.42 & 8.93 & 98.20 & 9.33 & 98.17 & 56.37 & 85.83 & 41.76 & 91.50 & 26.47 & 94.48 & 94.03 \\ 
GODIN~\cite{hsu2020generalized} & 3.78 & 99.17 & 9.47 & 97.83 & 5.40 & 98.67 & 6.73 & 98.54 & 23.90 & 94.32 & 55.24 & 86.52 & 17.42 & 95.84 & 94.22  \\
Energy Score~\cite{liu2020energy} & 33.07 & 95.01 & 8.10 & 98.43 & 14.04 & 97.47 & 14.58 & 97.42 & 59.61 & 85.42 & 41.98 & 91.48 & 28.40 & 94.22 & 94.03 \\

ReAct~\cite{sun2021react} & 43.67 & 94.27 & 22.37 & 96.22 & 16.68 & 97.09 & 19.81 & 96.74 & 53.44 & 89.63 & 43.23 & 91.88 & 33.12 & 94.32 & 93.27\\

GradNorm~\cite{huang2021importance} & 25.07 & 93.91 & 0.41 & 99.85 & 9.51 & 98.08 & 10.41 & 97.97 & 44.72 & 83.23 & 58.65 & 82.45 & 24.79 & 92.58 & 94.03 \\
LogitNorm~\cite{wei2022mitigating} & 14.31 & 97.63 & 2.61 & 99.37 & 17.16 & 97.18 & 17.14 & 97.16 & 39.66 & 91.17 & 47.30 & 90.40 & 19.61 & 95.51 & 93.94\\ 
DICE~\cite{sun2022dice} & 27.84 & 94.98 & 0.38 & 99.90 & 4.43 & 99.03 & 5.14 & 98.97 & 45.85 & 86.97 & 45.41 & 90.03 & 20.83 & 95.24 & 94.38 \\
\midrule
\emph{Methods using feature representations}\\
Mahalanobis~\cite{lee2018simple} & 17.85 & 94.66 & 68.49 & 76.21 & 30.06 & 92.16 & 29.86 & 91.15 & 30.73 & 88.83 & 90.34 & 52.37  & 44.55 & 82.56 & 94.03 \\ 

KNN~\cite{sun2022knn} & 3.87 & 99.31 & 10.81 & 98.13 & 12.58 & 97.75 & 12.24 & 97.87 & 21.61 & 96.07 & 49.36 & 89.54 & 18.50 & 96.36 & 94.03\\
\midrule
 \name (Ours) & 2.67 & 99.52 & 5.22 & 99.14 & 9.70 & 98.35 & 8.94 & 98.44 & 19.84 & 96.51 & 43.62 & 90.98 & \textbf{15.00} & \textbf{97.16} & 94.15 \\
\bottomrule

\end{tabular}}
\caption{table}{\small Detailed results on six OOD benchmark datasets: \texttt{Textures}~\cite{cimpoi2014describing}, \texttt{SVHN}~\cite{svhn}, \texttt{LSUN-Crop}~\cite{yu2015lsun}, \texttt{LSUN-Resize}~\cite{yu2015lsun}, \texttt{iSUN}~\cite{xu2015turkergaze}, and \texttt{Places365}~\cite{zhou2017places}. Model is DenseNet and ID is CIFAR-10.}
\label{tab:ablation_complete_c10}
\end{table*}
\begin{table*}[h!]
\centering
\resizebox{\textwidth}{!}{%
\begin{tabular}{l*{16}c}
 \multirow{2}{1.5cm}{\textbf{Methods}} & \multicolumn{12}{c}{\textbf{OOD Datasets}} & \multicolumn{2}{c}{\multirow{2}{*}{\centering\textbf{Average}}} & \multirow{2}{*}{\centering\textbf{ID Acc.}}\\

\cmidrule{2-13}

& \multicolumn{2}{c}{\textbf{SVHN}} & \multicolumn{2}{c}{\textbf{LSUN-c}} &
\multicolumn{2}{c}{\textbf{LSUN-r}} &
\multicolumn{2}{c}{\textbf{iSUN}} & \multicolumn{2}{c}{\textbf{Textures}} & \multicolumn{2}{c}{\textbf{Places365}} & &  \\

& FPR95 & AUROC & FPR95 & AUROC & FPR95 & AUROC &  FPR95 & AUROC & FPR95 & AUROC & FPR95 & AUROC & FPR95 & AUROC & \\
 & $\pmb{\downarrow}$ & $\pmb{\uparrow}$  & $\pmb{\downarrow}$ & $\pmb{\uparrow}$ & $\pmb{\downarrow}$ &  $\pmb{\uparrow}$ & $\pmb{\downarrow}$ & $\pmb{\uparrow}$ & $\pmb{\downarrow}$ & $\pmb{\uparrow}$ & $\pmb{\downarrow}$ & $\pmb{\uparrow}$ &
 $\pmb{\downarrow}$ & $\pmb{\uparrow}$ & $\pmb{\uparrow}$  \\
\midrule
\emph{Methods using Model Outputs}\\

MSP~\cite{hendrycks2016baseline} & 83.67 & 75.46 & 61.00 & 86.00 & 74.73 & 76.13 & 76.10 & 75.48 & 86.17 & 71.65 & 83.31 & 73.97 & 77.59 & 76.47 & 75.14 \\
ODIN~\cite{liang2018enhancing} & 91.51 & 76.16 & 15.16 & 97.41 & 31.92 & 93.93 & 36.75 & 92.89 & 83.92 & 72.70 & 79.12 & 77.13 & 56.39 & 86.02 & 75.14 \\
GODIN~\cite{hsu2020generalized} & 15.25 & 97.15 & 30.65 & 93.66 & 42.75 & 93.02 & 38.50 & 93.53 & 47.98 & 89.62 & 89.37 & 70.23 & 44.08 & 89.05 & 74.22 \\
Energy Score~\cite{liu2020energy} & 87.94 & 78.07 & 13.81 & 97.61 & 35.82 & 92.98 & 40.75 & 91.78 & 84.38 & 71.81 & 79.91 & 76.71 & 57.07 & 84.83 & 75.14\\
ReAct~\cite{sun2021react} & 93.65 & 74.20 & 52.07 & 87.63 & 63.14 & 88.13 & 69.96 & 85.56 & 87.07 & 72.56 & 87.90 & 67.66 & 75.06 & 79.51 & 66.56 \\ 
GradNorm~\cite{huang2021importance} & 60.62 & 87.76 & 0.65 & 99.78 & 82.20 & 75.48 & 78.68 & 78.14 & 65.73 & 71.99 & 90.41 & 65.65 & 63.05 & 79.80 & 75.14 \\
LogitNorm~\cite{wei2022mitigating} & 57.65 & 89.32 & 12.37 & 97.76 & 70.53 & 84.94 & 71.27 & 84.54 & 74.91 & 75.20 & 78.00  & 78.42 & 61.10 & 84.72 & 75.42 \\
DICE~\cite{sun2022dice} & 59.80 & 88.29 & 0.91 & 99.74 & 51.62 & 89.32  & 49.48 & 89.51 & 61.42 & 77.12 & 80.27 & 77.40 & 49.72 & 87.23 & 68.65 \\
\midrule
\emph{Methods using feature representations}\\
Mahalanobis~\cite{lee2018simple} & 70.19 & 80.49 & 93.98 & 66.81 & 24.83 & 94.97 & 26.20 & 94.19 & 31.76 & 90.01 & 94.60 & 55.17 & 56.93 & 80.27 & 75.14\\
KNN~\cite{sun2022knn} & 23.54 & 95.34 & 66.59 & 77.98 & 37.83 & 92.88 & 32.83 & 93.63 & 28.58 & 92.36 & 93.92 & 59.42 & 47.21 & 85.27  & 75.14 \\
\midrule
 \name (Ours) & 11.56 & 97.68 & 24.43 & 94.44 & 19.19 & 96.17 & 21.21 & 95.46 & 22.93 & 94.75 & 88.17 & 66.58 & \textbf{31.25} & \textbf{90.85} & 75.59 \\
\bottomrule

\end{tabular}}
\caption{table}{\small Detailed results on six OOD benchmark datasets: \texttt{Textures}~\cite{cimpoi2014describing}, \texttt{SVHN}~\cite{svhn}, \texttt{LSUN-Crop}~\cite{yu2015lsun}, \texttt{LSUN-Resize}~\cite{yu2015lsun}, \texttt{iSUN}~\cite{xu2015turkergaze}, and \texttt{Places365}~\cite{zhou2017places}. Model is DenseNet and ID is CIFAR-100.}
\label{tab:ablation_complete_c100}
\end{table*}
\end{document}